\documentclass[11pt]{article}

\usepackage[preprint]{acl}

\usepackage{times}
\usepackage{latexsym}
\usepackage[T1]{fontenc}
\usepackage[utf8]{inputenc}
\usepackage{microtype}
\usepackage{inconsolata}

\usepackage{graphicx}
\usepackage{subcaption}
\usepackage{booktabs}
\usepackage{amsmath,amssymb,mathtools}
\usepackage{amsthm}
\usepackage{multirow}
\usepackage[table]{xcolor}
\usepackage{dblfloatfix}
\usepackage{algorithm}
\usepackage{algpseudocode}
\usepackage{enumitem}
\usepackage[capitalize,noabbrev]{cleveref}

\allowdisplaybreaks

\theoremstyle{plain}
\newtheorem{theorem}{Theorem}[section]

\newtheorem{lemma}[theorem]{Lemma}
\newtheorem{corollary}[theorem]{Corollary}
\theoremstyle{definition}

\theoremstyle{remark}
\newtheorem{remark}[theorem]{Remark}


\definecolor{ccpoblue}{RGB}{25, 95, 160}

\title{Counterfactual Credit Policy Optimization for Multi-agent Collaboration}

\author{
	\textbf{Zhongyi Li}\textsuperscript{1}
	\and
	\textbf{Wan Tian}\textsuperscript{2}
	\and
	\textbf{Jinju Chen}\textsuperscript{3}
	\and 
	\textbf{Huiming Zhang}\textsuperscript{1}
	\and
	\textbf{Yang Liu}\textsuperscript{2}
	\\
	\textbf{Yikun Ban}\textsuperscript{1}\thanks{Corresponding authors.}
	\and
	\textbf{Fuzhen Zhuang}\textsuperscript{1}\footnotemark[1]
	\\
	\textsuperscript{1}Beihang University
	\quad
	\textsuperscript{2}Peking University
	\quad
	\textsuperscript{3}Beijing University of Posts and Telecommunications
	\\[0.8ex]
	\small\textcolor{ccpoblue}{\textbf{Project Page:}}~\url{https://bhai114.github.io/ccpo_page/}
}

\definecolor{bgblue}{RGB}{220, 230, 245}

\begin{document}
	\maketitle

	\begin{abstract}
		Collaborative multi-agent large language models (LLMs) can solve complex reasoning tasks by decomposing roles, but reinforcement learning for such systems is limited by credit assignment: shared terminal rewards obscure individual contributions and can encourage free-riding. We introduce two optimizer-agnostic credit assignment methods for converting joint outcomes into agent-specific learning signals. \textbf{Counterfactual Credit for Policy Optimization (CCPO)} estimates an agent's marginal contribution by comparing the realized joint outcome with a counterfactual outcome where that agent is removed. \textbf{Self-Evaluated Credit for Policy Optimization (SEPO)} uses constrained self- and peer-evaluations as a verifier-anchored credit signal while keeping the external task outcome dominant. Both operate at the reward-construction layer rather than as policy optimizers, producing role-specific rewards or advantages for GRPO, GSPO, or REINFORCE++. We instantiate these credit signals in a sequential Think--Solve setting and evaluate them on mathematical reasoning benchmarks. Results show that explicit credit assignment often improves dual-agent reasoning, especially on MATH500 and several out-of-distribution settings, while gains vary across models and datasets. Our code is available at: \url{https://github.com/bhai114/ccpo}.

	\end{abstract}

	\section{Introduction}
	
	LLMs have become increasingly capable on complex reasoning, mathematical problem solving, and code generation \cite{li2025system,duan2025codeplot,lyu2025automatic,he2025llm}. These capabilities make LLMs promising building blocks for systems that decompose difficult tasks, assign specialized roles, and combine intermediate reasoning into final decisions \cite{ban2026epistemic}. Yet single-model inference remains brittle on long-horizon problems, where exploration, intermediate verification, and self-correction are limited. Multi-agent LLM  collaboration is therefore an important direction \cite{chen2026weak,zhang2026heterogeneous}: by letting agents play complementary roles, such as a \emph{Thinker} that proposes reasoning and a \emph{Solver} that produces the final answer, collaborative systems can improve reliability at inference time and may also learn stronger cooperative behaviors when trained jointly \cite{tran2025multi}.
	
	Training collaborative LLM systems remains underdeveloped because credit assignment is unresolved \cite{chen2025harnessing,lin2025speaking,nagpal2025leveraging}: a sparse and delayed joint outcome must be attributed to heterogeneous agents that generate long, discrete text trajectories. Existing multi-agent RL practice often relies on shared global rewards or value-decomposition surrogates developed mainly for small-scale continuous-control settings \cite{jiang2025qllm}, but these solutions are poorly matched to LLM reasoning. Shared rewards cannot identify which message helped or harmed the final answer, can reinforce redundant or detrimental behavior, and ignore role asymmetry. Counterfactual credit assignment, including counterfactual baselines \cite{foerster2018counterfactual} and Shapley-style marginalization, is conceptually appealing but costly for long textual trajectories. The limitation is visible on MATH500 \cite{lightman2023letsverifystepstep}: an agent's marginal contribution can vary substantially across model scales, and in some pairings a Solver can outperform the full collaboration when answering alone. Under a shared terminal reward, both agents still receive the same learning signal even when one agent contributes little or harms the joint outcome.
	
	\begin{figure*}[t]
		\centering
		\includegraphics[width=0.99\textwidth]{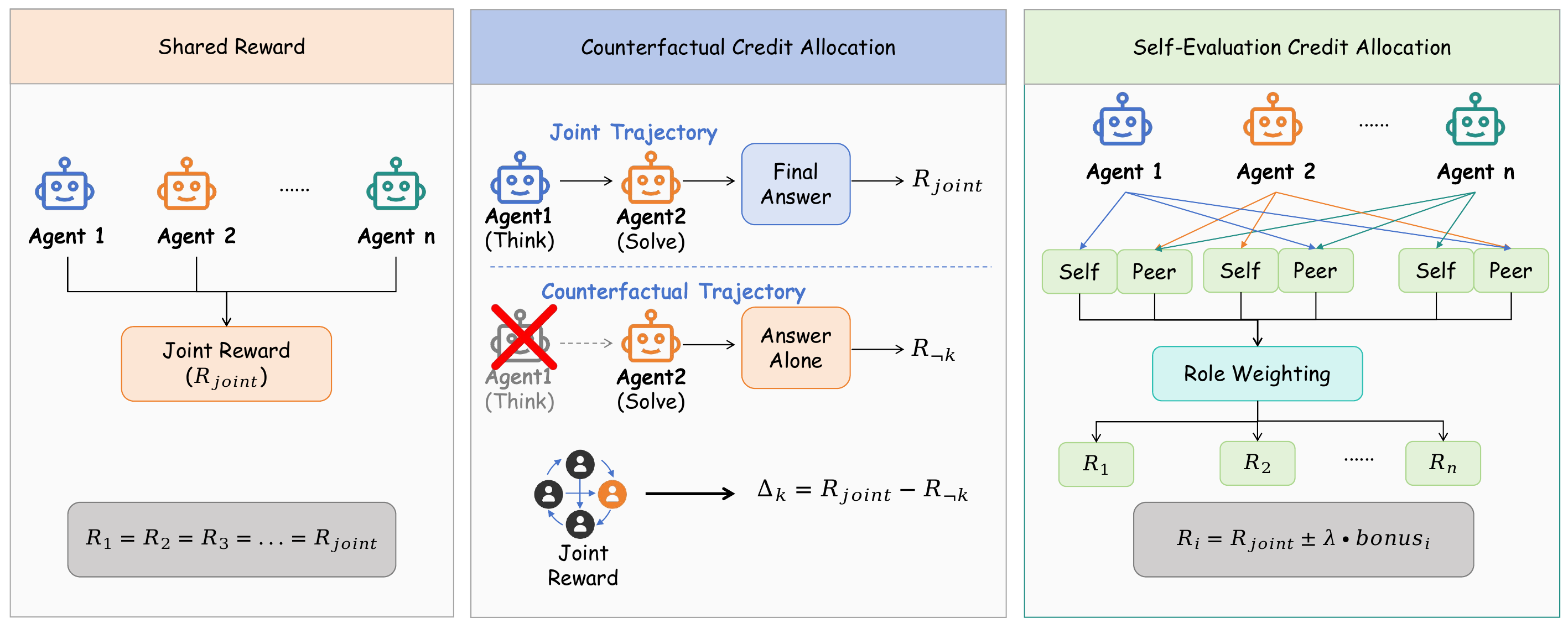}
		\caption{Overview of the reward-construction layer. Shared rewards, CCPO, and SEPO define credit signals that are passed to a separate policy optimizer.}
		\label{fig:framework}
	\end{figure*}
	
	We study credit assignment as a first-class interface between collaborative rollouts and policy optimization. Our goal is to produce role-specific learning signals that remain anchored to the external verifier, are sensitive to each agent's marginal influence, and are compatible with modern sequence-level optimizers. To this end, we propose two optimizer-agnostic credit modules, illustrated in \cref{fig:framework}. CCPO asks how the joint outcome would change if one agent's contribution were removed while the remaining agents are kept fixed, yielding a marginal credit signal for collaborative training. SEPO uses constrained self- and peer-evaluations as bounded adjustments around an external verifier outcome, rather than replacing verification with self-judgment. CCPO and SEPO determine what each agent is credited or blamed for; optimizers such as GRPO \cite{shao2024deepseekmathpushinglimitsmathematical}, GSPO \cite{zheng2025group}, and REINFORCE++ \cite{hu2025reinforcestabilizingcriticfreepolicy} determine how policies are updated from those signals. We instantiate these credit modules in a sequential Think--Solve topology, where the Thinker generates intermediate reasoning and the Solver produces the final answer, making asymmetric contribution explicit.
	
	Our contributions are summarized as follows:
	\begin{itemize}[leftmargin=*, topsep=2pt, itemsep=2pt, parsep=0pt, partopsep=0pt]
		\item We formulate role-sensitive credit assignment for collaborative LLM training and introduce CCPO as a lightweight counterfactual credit module against shared final-reward training. We show that, under conditional-independence assumptions, CCPO can be interpreted as a valid baseline-subtracted policy-gradient signal and may reduce gradient variance compared with shared rewards.
		\item We introduce SEPO as a verifier-anchored self-evaluation credit module that uses structured self- and peer-evaluations as bounded credit adjustments.
		\item Empirically, CCPO rewards improve in-distribution MATH500 performance across the evaluated base models and provide gains on several out-of-distribution benchmarks, while SEPO rewards are competitive in selected GSPO settings.
	\end{itemize}



	\section{Problem Setup}

	We consider a cooperative system with $K$ LLM agents $U=\{u_1,\ldots,u_K\}$, where agent $u_i$ follows a stochastic policy $\pi_{\theta_i}$.
	Given a prompt $x \sim \mathcal{D}$, agents generate textual outputs; depending on the collaboration topology, an agent's generation may condition on other agents' outputs.
	A task-specific evaluator returns a bounded scalar reward $R(\cdot)\in[-1,+1]$ (e.g., exact-match correctness on reasoning benchmarks).
	
	For each prompt $x$, we sample $N$ joint rollouts. The $j$-th rollout is denoted by
	$
	\tau^{(j)} \;=\; \big(y_1^{(j)},\ldots,y_K^{(j)}\big),
	R_{\mathrm{joint}}^{(j)} \;:=\; R\!\big(\tau^{(j)}\big).
	$
	Our goal is to optimize each $\pi_{\theta_i}$ using \emph{agent-specific} learning signals (e.g., advantages) that reflect agent $u_i$'s marginal contribution to the joint reward $R_{\mathrm{joint}}$.
	
	In \emph{sequential} collaboration, agents act in a fixed order and pass intermediate text forward.
	Given $x$, the first agent generates
	$
	y_1 \sim \pi_{\theta_1}(\cdot \mid x),
	$
	and for $i=2,\ldots,K$, agent $u_i$ generates conditioned on the prompt and all previous outputs,
	$
	y_i \sim \pi_{\theta_i}\big(\cdot \mid x, y_{1:i-1}\big), \text{where } y_{1:i-1} := (y_1,\ldots,y_{i-1}).
	$
	The final team output is taken to be the last agent's output,
	$
	\widehat{y} \;:=\; y_K,
	$
	and the joint reward is computed as $R_{\mathrm{joint}} = R(x, \widehat{y})$ (or more generally $R_{\mathrm{joint}}=R(x, y_{1:K})$ if the evaluator uses the full transcript).
	


	\section{Credit Assignment and Reward Construction}
	\label{sec:ccpo_method}
	
	CCPO and SEPO separate credit allocation from policy optimization. Given a collaborative rollout, a verifier or task evaluator produces the joint outcome, and a credit-assignment module maps that outcome into role-specific rewards. These rewards can then be normalized into advantages and passed to GRPO, GSPO, REINFORCE++, or other policy-gradient optimizers. In the Think--Solve setting used in this paper, the Thinker generates an intermediate reasoning trace $y_1$, and the Solver produces the final answer $y_2$ conditioned on $(x,y_1)$.
	
	\subsection{Shared Joint Reward}
	The simplest baseline assigns the same terminal reward to every agent:
	\[
	r_1^{(j)} = r_2^{(j)} = R_{\mathrm{joint}}^{(j)}.
	\]
	This baseline is easy to optimize, but it does not distinguish which role helped or hurt the final answer.
	
	\subsection{CCPO}
	Counterfactual credit asks how the joint outcome changes when one agent is removed while the others stay fixed. For agent $i$ and rollout $j$, let the realized joint reward be $R_{\mathrm{joint}}^{(j)}$, and let $R_{\neg i}^{(j)}$ be the counterfactual reward after removing agent $i$. We use $R_{\neg i}^{(j)}$ as the canonical notation throughout the paper; in the Think--Solve instantiation, the solver-only reward is the special case $R_{\neg 1}^{(j)}=R_{\mathrm{solo}}^{(j)}$. The marginal contribution is
	\[
	\Delta_i^{(j)} = R_{\mathrm{joint}}^{(j)} - R_{\neg i}^{(j)}.
	\]
	Positive values indicate that agent $i$ improves the joint outcome; non-positive values indicate that the agent is redundant or harmful under that rollout.
	
	This makes counterfactual credit fundamentally different from shared rewards. Under a shared terminal signal, all agents are updated as if they contributed equally to the final answer, even when one role is mostly carrying the collaboration. In contrast, $\Delta_i^{(j)}$ isolates the agent's marginal effect relative to what the remaining agents can already achieve, so the resulting update is role-sensitive and explicitly discourages free-riding. From an optimization perspective, the counterfactual term can be viewed as a baseline-subtracted return that preserves the external task outcome as the anchor while redistributing credit to the agent whose presence actually changes the result.
	
	In Think--Solve, removing the Thinker means asking the Solver to answer directly from the prompt. This gives a concrete comparison between the collaborative rollout $(y_1^{(j)}, y_2^{(j)})$ and a solver-only rollout, so the Thinker is rewarded for genuinely improving the final answer rather than for merely participating in the trajectory. In this instantiation, the solver-only counterfactual is sampled independently of the current Thinker output, which is the condition required for interpreting $R_{\neg 1}^{(j)}$ as an action-independent baseline for the Thinker in the analysis. We do not symmetrically define a leave-one-out final-answer counterfactual for the Solver, because removing the Solver would leave no final answer in this topology; instead, the Solver uses the fused team/solo signal specified in Appendix~\ref{app:dual_agent_details}. Other collaboration topologies may support counterfactuals for additional roles, but each counterfactual must be constructed so that it does not depend on the removed agent's sampled action in the current rollout. Compared with coalitional or Shapley-style attribution, this design is lightweight because it avoids enumerating agent subsets and relies on a small number of counterfactual evaluations per prompt.
	In the implementation, $\Delta_i^{(j)}$ is treated as the raw credit signal and then passed through the shaping and normalization steps described in Appendix~\ref{app:ccpo_details}. This keeps CCPO compatible with long sequences, heterogeneous reward scales, and optimizer-agnostic training pipelines such as GRPO, GSPO, or REINFORCE++. Importantly, the raw reward scale need not be identical across credit allocators: binary $[0,1]$ task rewards, counterfactual margins, and $\{-1,+1\}$ verifier outcomes are converted into role-specific rewards and then normalized within the optimizer-specific advantage pipeline.
	
	\subsection{SEPO}
	In the two-agent Think--Solve setting, SEPO uses LLM-generated self and peer assessments to allocate credit. An external verifier first scores the Solver's final answer as
	\[
	R_{\mathrm{ver}}^{(j)} \in \{-1,+1\}.
	\]
	The Thinker and Solver each output constrained self and peer scores:
	\[
	(p_i^{\mathrm{self},(j)}, p_i^{\mathrm{peer},(j)}) \in \mathcal{V}\times \mathcal{V},
	\qquad i\in\{1,2\},
	\]
	where $\mathcal{V}$ is a finite ordered rubric of preset levels. The fused scores are
	\begin{align}
		s_1^{(j)} &= \eta\,p_1^{\mathrm{self},(j)} + (1-\eta)\,p_2^{\mathrm{peer},(j)}, \\
		s_2^{(j)} &= \eta\,p_2^{\mathrm{self},(j)} + (1-\eta)\,p_1^{\mathrm{peer},(j)}.
	\end{align}
	We then normalize them into role weights:
	\[
	w_i^{(j)} = \frac{s_i^{(j)}}{s_1^{(j)} + s_2^{(j)} + \epsilon}, \qquad i\in\{1,2\}.
	\]
	If group centering is enabled, we use
	\[
	\mathrm{bonus}_i^{(j)} = w_i^{(j)} - \operatorname{mean}_{j' \in \mathcal{G}(x)}\!\bigl(w_i^{(j')}\bigr).
	\]
	Finally, the SEPO role reward is
	\[
	r_i^{(j)} =
	\begin{cases}
		R_{\mathrm{ver}}^{(j)} + \lambda_{\mathrm{credit}}\,\mathrm{bonus}_i^{(j)}, & R_{\mathrm{ver}}^{(j)} = +1,\\
		R_{\mathrm{ver}}^{(j)} - \lambda_{\mathrm{blame}}\,\mathrm{bonus}_i^{(j)}, & R_{\mathrm{ver}}^{(j)} = -1.
	\end{cases}
	\]
	with default values $\lambda_{\mathrm{credit}}=\lambda_{\mathrm{blame}}=0.2$.
	SEPO keeps the verifier outcome as the dominant signal and uses self and peer evaluations only as a bounded credit-allocation adjustment. It is therefore not a replacement for task verification: if the final answer is wrong, the base signal remains negative and the self/peer scores only redistribute responsibility around that outcome. Because these scores are produced by LLMs, we treat SEPO as an empirical, complementary allocator rather than a theoretically guaranteed credit signal.

	\section{Theoretical Analysis of Counterfactual Credit}
	\label{sec:ccpo_theory}
	
	
	This section analyzes CCPO. SEPO is intentionally bounded to serve as a credit-redistribution signal around the verifier outcome; we therefore treat it empirically rather than claiming separate optimization guarantees for it. The main theoretical point is modest: under explicit conditional-independence assumptions, counterfactual rewards can be interpreted as valid baseline-subtracted returns for policy-gradient estimation.
	
	Our main implementation uses GRPO as the base optimizer, which applies clipped importance ratios and optionally KL monitoring to keep each policy update conservative.
	Other policy-gradient optimizers can consume the same CCPO rewards, but the monotonicity statement below is tied to conservative trust-region-style updates rather than to the reward allocator itself.
	Accordingly, we present an idealized monotonic-improvement characterization under a KL trust-region condition; this result should be read as a sanity check for conservative block updates, not as a convergence guarantee for practical GRPO training. The proof
	is deferred to Appendix~\ref{app:ccpo_convergence}.
	\begin{theorem}
		\label{prop:ccpo_tr_mono}
		Consider alternating updates where one agent $k$ is updated while the other agents are fixed at $\theta_{-k}^t$.
		Let $\pi_{\text{old}}:=\pi_{\theta_k^{t}}$ and $\pi_{\text{new}}:=\pi_{\theta_k^{t+1}}$ be the policies before and after
		one block update for agent $k$. Assume the update is conservative in the sense that it satisfies a KL trust-region bound
		\begin{equation}
			\label{eq:prop_eff_tr}
			D_{\mathrm{KL}}^{\max}(\pi_{\text{old}},\pi_{\text{new}})\le \delta,
		\end{equation}
		which practical clipped updates can only approximate and monitor rather than enforce exactly.
		Assume further that the update improves a TRPO-style surrogate objective by at least $\Delta_L>0$ in the induced
		stationary MDP with other agents fixed, and that the old-policy advantage is bounded by
		$\epsilon=\max_{s,a}|A_{\pi_{\text{old}}}(s,a)|$.
		Then
		\begin{equation}
			\label{eq:main_trpo_gain}
			J(\theta_k^{t+1},\theta_{-k}^{t}) - J(\theta_k^{t},\theta_{-k}^{t})
			\ \ge\
			\Delta_L - C(\gamma)\,\epsilon\,\sqrt{\delta},
		\end{equation}
		where \(C(\gamma)=\frac{2\gamma\sqrt{2}}{(1-\gamma)^2}\). In particular, if $\Delta_L \ge C(\gamma)\epsilon\sqrt{\delta}$ for each block update, then $J$ is non-decreasing and,
		since $J\in[0,1]$, the sequence of objective values converges.
	\end{theorem}

	\begin{remark}
		Theorem~\ref{prop:ccpo_tr_mono} is an idealized characterization: it assumes a KL-bounded trust region and a positive surrogate improvement for each block update.
		GRPO, GSPO, and PPO-like implementations approximate this condition through clipping or KL control, but clipping is not equivalent to an exact KL constraint and does not guarantee strict monotonic improvement in general.
		The theorem is therefore not a practical convergence guarantee for CCPO; it clarifies the conditions under which any conservative block update with CCPO rewards would inherit a standard trust-region lower bound.
		A formal derivation under the KL trust-region condition is also provided in Appendix~\ref{app:ccpo_convergence}.
	\end{remark}
	
	We next clarify why counterfactual credit is preferable to shared terminal rewards from a gradient-estimation perspective.
	Fix an active agent $k$ and hold the other agents $\theta_{-k}$ fixed.
	Let $R(\tau)\in[0,1]$ be the terminal joint reward for a joint rollout $\tau$, and let $R_{\neg k}$ be the counterfactual reward obtained by removing agent $k$ while keeping the remaining agents and the collaboration protocol unchanged.
	Define $\Delta_k := R(\tau)-R_{\neg k}$.
	Let
	\[
	g_k(\tau_k)\;:=\;\sum_t \nabla_{\theta_k}\log \pi_{\theta_k}(a_{k,t}\mid s_{k,t})
	\]
	denote agent $k$'s score-function term, where $(s_{k,t},a_{k,t})$ are token- or turn-level states/actions.
	
	Under the shared-reward baseline, agent~$k$ is trained with an estimator proportional to $g_k(\tau_k)\,R(\tau)$.
	In contrast, CCPO employs $g_k(\tau_k)\,\Delta_k$, which can be interpreted as subtracting an action-independent baseline $R_{\neg k}$ from the return.
	Theorem~\ref{prop:cf_vs_shared_main} below formalizes that such counterfactual baselines preserve unbiasedness and can reduce estimator variance only when the counterfactual is a better baseline than zero.

	\begin{theorem}
		\label{prop:cf_vs_shared_main}
		Assume that conditioned on $(x,\theta_{-k})$, the random variable $R_{\neg k}$ does not depend on agent $k$'s sampled actions in the joint rollout $\tau$.
		Then replacing $R(\tau)$ by $\Delta_k=R(\tau)-R_{\neg k}$ does not change $\nabla_{\theta_k}J(\boldsymbol{\theta})$ (no gradient bias). Moreover, consider the family of unbiased estimators of the form
		\[
		\widehat{G}_b \;=\; g_k(\tau_k)\,\bigl(R(\tau)-b\bigr),
		\]
		where $b$ is any scalar baseline measurable with respect to $(x,\theta_{-k})$ and independent of agent $k$'s sampled actions in $\tau$ (conditioned on $(x,\theta_{-k})$).
		Among all such baselines, the conditional variance
		$\mathrm{Var}(\widehat{G}_b\mid x,\theta_{-k})$
		is minimized by
		\begin{equation}
			\label{eq:opt_baseline_main}
			b^\star(x,\theta_{-k})
			\;=\;
			\frac{\mathbb{E}\!\left[\|g_k(\tau_k)\|_2^2\,R(\tau)\,\middle|\,x,\theta_{-k}\right]}
			{\mathbb{E}\!\left[\|g_k(\tau_k)\|_2^2\,\middle|\,x,\theta_{-k}\right]}.
		\end{equation}
		Consequently, the shared-reward estimator corresponds to the special case $b\equiv 0$, which is generally suboptimal unless $b^\star\equiv 0$. If $R_{\neg k}$ is closer to $b^\star$ than $0$ in the weighted mean-square sense, i.e.,
		$\mathbb{E}\!\left[\|g_k(\tau_k)\|_2^2\,(R_{\neg k}-b^\star)^2 \,\middle|\, x,\theta_{-k}\right]
		\le
		\mathbb{E}\!\left[\|g_k(\tau_k)\|_2^2\,(0-b^\star)^2 \,\middle|\, x,\theta_{-k}\right]$, then 
		\[\mathrm{Var}\!\left(g_k(\tau_k)\Delta_k\mid x,\theta_{-k}\right)\le \mathrm{Var}\!\left(g_k(\tau_k)R(\tau)\mid x,\theta_{-k}\right).
		\]
	\end{theorem}
	
	We do not assume this variance condition always holds in practice; it motivates the estimator but remains task- and counterfactual-quality dependent. In addition to this conditional variance-reduction view, $\Delta_k$ also provides a directional credit signal that discourages free-riding:
	whenever agent $k$ is redundant on a rollout so that $R(\tau)=R_{\neg k}$, CCPO assigns $\Delta_k=0$ and thus removes the spurious positive update that would arise under shared rewards.
	Formal statements and proofs for Theorem~\ref{prop:cf_vs_shared_main} are deferred to Appendix~\ref{app:ccpo_convergence}.


	\section{Experiments}
	\subsection{Experimental Setup}
	We evaluate CCPO and SEPO as credit-assignment modules in the two-agent Think--Solve topology from Section~\ref{sec:ccpo_method}: the Thinker generates an intermediate reasoning trace and the Solver produces the final answer. This setting directly tests whether a sparse final verifier signal can be converted into useful role-specific rewards. We train on MATH 7.5k \citep{hendrycks2021measuringmathematicalproblemsolving} and report exact-match accuracy on MATH500 (in-distribution), AIME25, AMC23, Gaokao2023en \citep{zhang2024evaluatingperformancelargelanguage}, and MinervaMath \citep{lewkowycz2022solvingquantitativereasoningproblems}.
	
	The main GRPO comparison uses the same protocol, data split, verifier, and base optimizer for an untrained collaborative policy, the shared-reward implementation of ReMA \cite{wan2503rema}, and GRPO trained with CCPO rewards. The GSPO study keeps the optimizer fixed and compares shared rewards, CCPO rewards, and SEPO rewards, so that differences come from the credit allocator rather than the optimizer. These experiments are not matched-compute comparisons because CCPO requires extra verifier calls. All experiments were run on 6 NVIDIA A800 GPUs; hyperparameters are provided in Table~\ref{tab:hyperparameters}.

	\begin{table}[!ht]
		\centering
		\caption{Dual-agent reasoning performance under the GRPO optimizer with different credit signals on mathematical benchmarks (Accuracy \%). } 
		\label{tab:main_results}
		\small
		\setlength{\tabcolsep}{0.6pt}
		\renewcommand{\arraystretch}{1.08}
		\begin{tabular}{@{}>{\centering\arraybackslash}p{0.13\linewidth}>{\centering\arraybackslash}p{0.25\linewidth}>{\centering\arraybackslash}p{0.18\linewidth}>{\centering\arraybackslash}p{0.14\linewidth}>{\centering\arraybackslash}p{0.26\linewidth}@{}}
			\toprule[1.1pt]
			\textbf{Model} & \textbf{Dataset} & \begin{tabular}[c]{@{}c@{}}\textbf{Untrained}\end{tabular} & \begin{tabular}[c]{@{}c@{}}\textbf{Shared}\end{tabular} & \begin{tabular}[c]{@{}c@{}}\textbf{CCPO}\end{tabular} \\ \midrule
			
			\multirow{5}{*}{\begin{tabular}[c]{@{}c@{}}qwen2.5\\1.5b\\instruct\end{tabular}} 
			& MATH500 & 54.00 & 60.00 & \textbf{61.00} \\
			& AIME25 & \textbf{-} & - & - \\
			& AMC23 & 32.50 & 20.00 & \textbf{35.00} \\
			& Gaokao2023 & 41.04 & \textbf{43.90} & 41.60 \\
			& MinervaMath & 12.50 & 13.60 & \textbf{14.70} \\
			\midrule
			\multirow{5}{*}{\begin{tabular}[c]{@{}c@{}}llama3.1\\8b\\instruct\end{tabular}} 
			& MATH500 & 46.20 & 51.80 & \textbf{53.40} \\
			& AIME25 & - & - & - \\
			& AMC23 & 22.50 & 25.00 & \textbf{35.00} \\
			& Gaokao2023 & 37.92 & 38.70 & \textbf{39.48} \\
			& MinervaMath & 17.28 & 16.54 & \textbf{20.22} \\ 
			\midrule
			\multirow{5}{*}{\begin{tabular}[c]{@{}c@{}}qwen2.5\\7b\\instruct\end{tabular}} 
			& MATH500 & 74.40 & 75.40 & \textbf{77.60} \\
			& AIME25 & 6.670 & \textbf{13.30} & 10.00 \\
			& AMC23 & 55.00 & \textbf{57.50} & 45.00 \\
			& Gaokao2023 & 57.40 & 57.70 & \textbf{59.74} \\
			& MinervaMath & 25.74 & 24.60 & \textbf{26.10} \\ 
			\midrule
			\multirow{5}{*}{\begin{tabular}[c]{@{}c@{}}qwen3\\4b\\base\end{tabular}} 
			& MATH500 & 46.40 & 78.00 & \textbf{79.40} \\
			& AIME25 & 6.670 & \textbf{13.33} & 6.670 \\
			& AMC23 & 25.00 & 42.50 & \textbf{52.50} \\
			& Gaokao2023 & 22.08 & \textbf{58.70} & 57.40 \\
			& MinervaMath & 5.510 & 26.84 & \textbf{27.94} \\ 
			\bottomrule[1.1pt]
		\end{tabular}
	\end{table}

	\begin{table}[t]
		\centering
		\caption{Reward-construction comparison under the fixed GSPO optimizer.}
		\label{tab:reward_comparison}
		\small
		\setlength{\tabcolsep}{1.0pt}
		\renewcommand{\arraystretch}{1.08}
		\begin{tabular}{@{}>{\centering\arraybackslash}p{0.17\linewidth}>{\centering\arraybackslash}p{0.22\linewidth}>{\centering\arraybackslash}p{0.11\linewidth}>{\centering\arraybackslash}p{0.13\linewidth}>{\centering\arraybackslash}p{0.19\linewidth}>{\centering\arraybackslash}p{0.15\linewidth}@{}}
			\toprule[1.1pt]
			\textbf{Model} 
			& \textbf{Dataset} 
			& \begin{tabular}[c]{@{}c@{}}\textbf{Un-}\\\textbf{trained}\end{tabular}
			& \begin{tabular}[c]{@{}c@{}}\textbf{Shared}\end{tabular}
			& \begin{tabular}[c]{@{}c@{}}\textbf{CCPO}\end{tabular}
			& \begin{tabular}[c]{@{}c@{}}\textbf{SEPO}\end{tabular} \\
			\midrule
			\multirow{5}{*}{\begin{tabular}[c]{@{}c@{}}qwen2.5\\1.5b\\instruct\end{tabular}}
			& MATH500      & 54.00 & 57.80   & \textbf{59.20}   & \textbf{59.20} \\
			& AIME25       & - & \textbf{6.670} & \textbf{6.670} & \textbf{6.670} \\
			& AMC23        & 32.50 & \textbf{42.50} & \textbf{42.50} & 35.00 \\
			& Gaokao2023   & 41.04 & 46.23 & \textbf{47.01} & 46.49 \\
			& MinervaMath  & 12.50 & 16.18 & 15.44 & \textbf{16.91} \\
			\midrule
			\multirow{5}{*}{\begin{tabular}[c]{@{}c@{}}olmo3\\7b\\instruct\end{tabular}}
			& MATH500      & 82.60 & \textbf{87.00} & 84.40 & \textbf{87.00} \\
			& AIME25       & \textbf{30.00} & \textbf{30.00} & \textbf{30.00} & \textbf{30.00} \\
			& AMC23        & 77.50 & 82.50 & 80.00 & \textbf{85.00} \\
			& Gaokao2023   & 72.73 & 73.25 & \textbf{73.77} & 73.25 \\
			& MinervaMath  & 27.57 & 27.91 & \textbf{30.15} & 26.47 \\
			\bottomrule[1.1pt]
		\end{tabular}
	\end{table}

	\begin{table}[t]
		\centering
		\caption{Performance (\%) of heterogeneous LLM collaboration under GSPO with different credit signals. GroupA denotes olmo3-7b-instruct\&qwen2.5-1.5b-instruct.}
		\label{tab:heterogeneous_gspo_results}
		\small
		\setlength{\tabcolsep}{3pt}
		\renewcommand{\arraystretch}{1.08}
		\resizebox{\linewidth}{!}{
			\begin{tabular}{@{}llcccc@{}}
				\toprule[1.1pt]
				\textbf{Model} & \textbf{Dataset} & \textbf{Untrained} & \textbf{Shared} & \textbf{CCPO} & \textbf{SEPO} \\
				\midrule
				\multirow{5}{*}{GroupA} & MATH500 & 73.60 & 87.20 & 87.00 & \textbf{87.40} \\
				& AIME25 & 20.00 & 33.33 & \textbf{36.67} & 33.33 \\
				& AMC23 & 52.50 & 72.50 & 80.00 & \textbf{82.50} \\
				& Gaokao2023en & 55.32 & \textbf{69.09} & 68.31 & \textbf{69.09} \\
				& MinervaMath & 20.22 & 27.94 & \textbf{30.15} & 25.37 \\
				\bottomrule[1.1pt]
			\end{tabular}
		}
	\end{table}
	
	\begin{table}[h]
		\centering
		\caption{Ablation of the Think--Solve handoff.} 
		\resizebox{\linewidth}{!}{
			\renewcommand{\arraystretch}{1}
			\begin{tabular}{@{}lcc@{}}
				\toprule[1.1pt]
				\textbf{Model} & \textbf{Full collaboration} & \textbf{Agent 1 removed} \\
				\midrule
				qwen2.5-1.5b-instruct & 61.00  & 56.40 \\
				llama3.1-8b-instruct  &  53.40     & 52.00      \\
				qwen2.5-7b-instruct   &  77.60     & 76.00      \\
				\bottomrule[1.1pt]
			\end{tabular}
		}
		\label{tab:handoff_ablation}
	\end{table}
	
	\subsection{GRPO with Counterfactual Credit}
	\cref{tab:main_results} shows that counterfactual credit improves over shared reward on MATH500 for all reported base models and is often beneficial on OOD benchmarks such as AMC23 and MinervaMath. This suggests that estimating the Thinker's marginal contribution can help when the two roles provide separable information. The gains are not uniform: shared reward remains competitive on several AIME25, AMC23, and Gaokao2023 settings, especially when the Solver can already solve many examples from the prompt alone. We therefore view the results as evidence that role-specific credit is useful for collaborative training, rather than as a claim that one allocator dominates every model--dataset pair.

	\subsection{GSPO with Different Credit Signals}
	
	\cref{tab:reward_comparison} shows that CCPO and SEPO rewards can both be used by the same GSPO optimizer. SEPO rewards are competitive for qwen2.5-1.5b-instruct and perform best on AMC23 for olmo3-7b-instruct, while CCPO rewards are strongest on Gaokao2023 and MinervaMath for olmo3-7b-instruct. This supports the optimizer-agnostic view of the proposed credit signals: once the joint outcome is mapped to role-specific rewards, the resulting signal can be consumed by different policy-gradient optimizers. Overall, with GSPO held fixed, both counterfactual and self-evaluation-based credit assignment outperform the shared-reward baseline in terms of accuracy. These gains are consistent across most mathematical reasoning scenarios.
	\begin{figure*}[!t]
		\centering
		\includegraphics[width=0.99\textwidth]{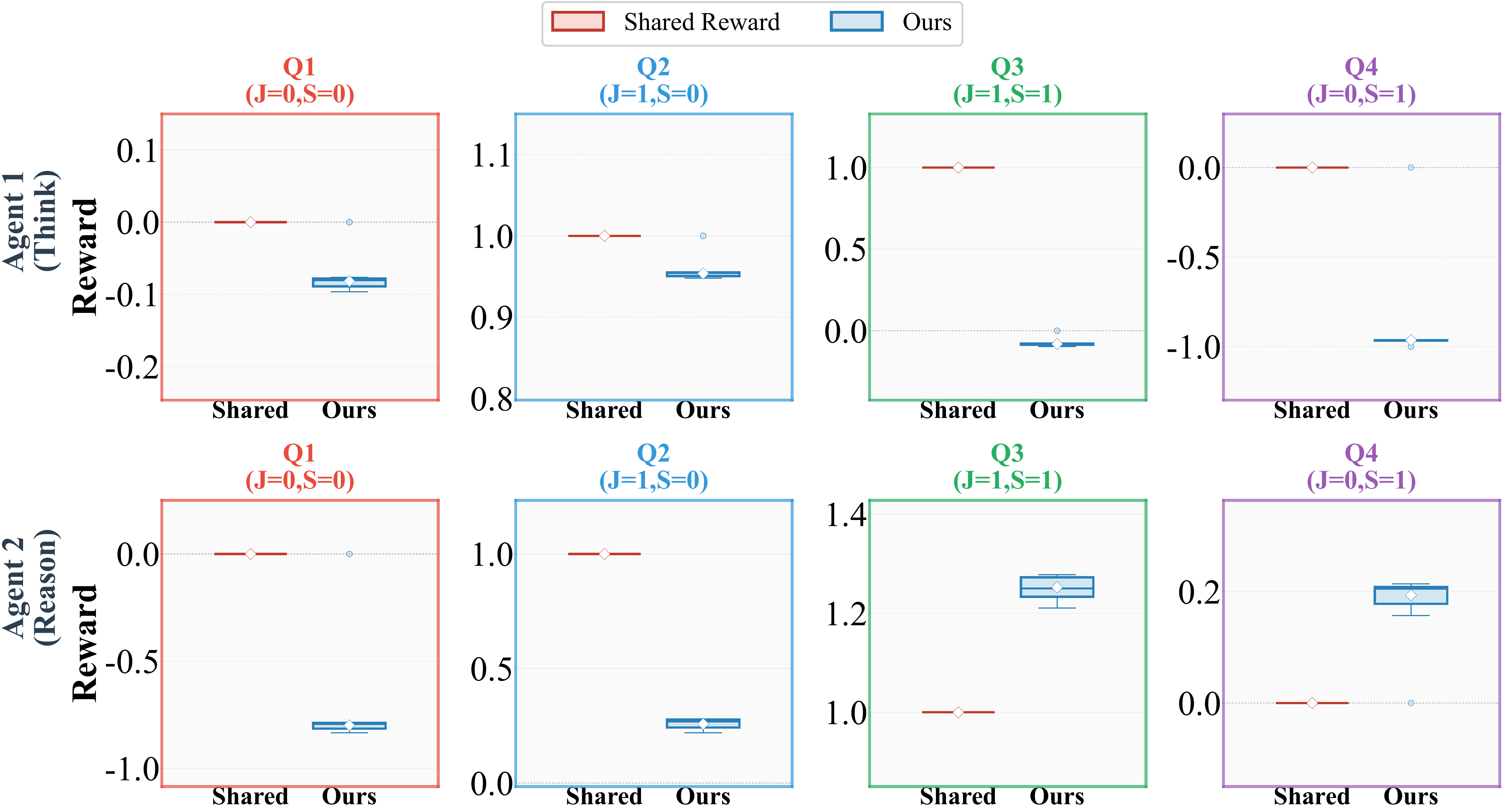}
		\caption{Reward distributions under shared rewards and CCPO credit assignment. The first row shows the Thinker and the second row shows the Solver; CCPO rewards separate contribution patterns more clearly than shared rewards (J: joint-answer score; S: Solver-only score).}
		\label{fig:boxPic}
		\vspace{-2mm}
	\end{figure*}
	\subsection{Performance on Heterogeneous LLMs}
	
	We further evaluate the credit signals in a heterogeneous Think--Solve setting trained with GSPO, where the two agents use different base LLMs. \cref{tab:heterogeneous_gspo_results} shows that explicit credit assignment remains competitive in this setting. Self-evaluation performs best on MATH500 and AMC23, while counterfactual credit is strongest on AIME25 and MinervaMath.

	\subsection{Collaboration and Credit Diagnostics}
	\subsubsection{Think--Solve Handoff}
	To check whether the Solver uses the Thinker's message, we remove Agent~1 at inference time. \cref{tab:handoff_ablation} shows that full collaboration is better for all reported models, suggesting that the learned Solver still benefits from the handoff. This diagnostic helps rule out the degenerate case in which training only improves a standalone Solver that ignores the collaborative trace.
	
	\subsubsection{Computational Trade-off}
	\begin{figure}[htbp]
		\centering
		\includegraphics[width=0.5\textwidth]{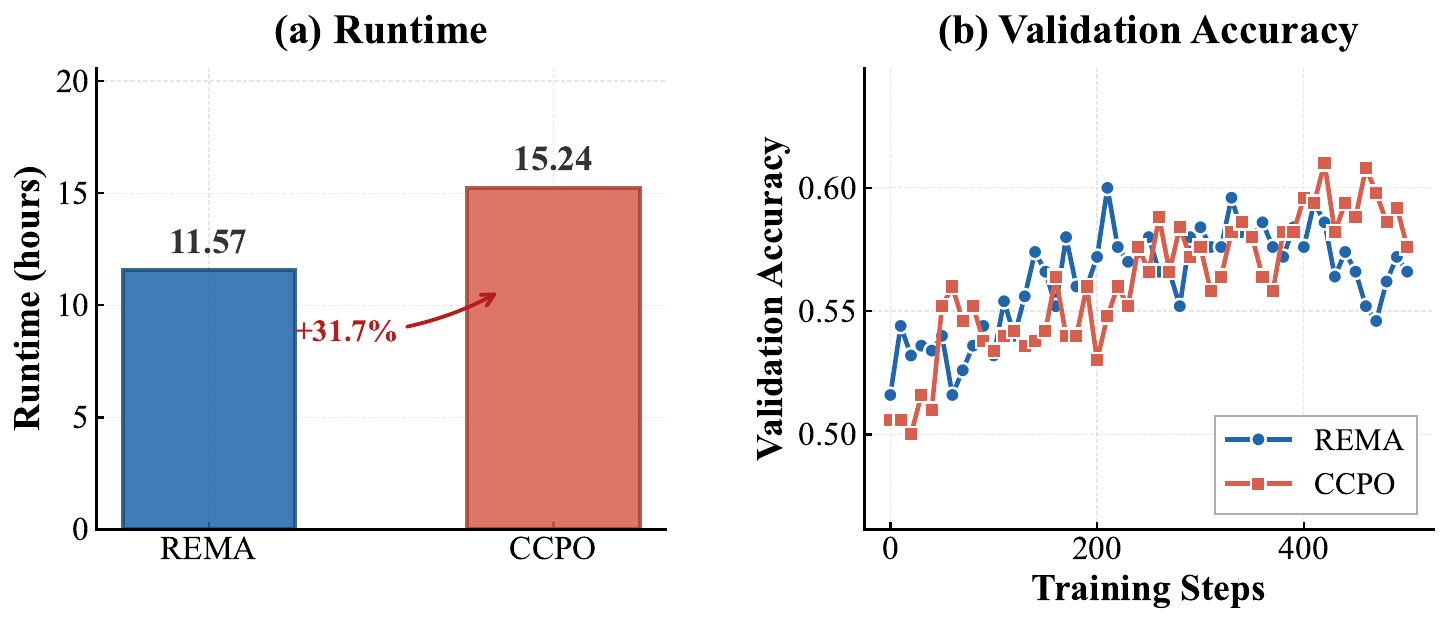}
		\caption{Training efficiency and validation accuracy of GRPO with CCPO rewards versus GRPO with shared rewards.}
		\label{fig:compute_tradeoff}
	\end{figure}
	
	CCPO requires extra verifier calls to estimate leave-one-role-out outcomes. \cref{fig:compute_tradeoff} reports the resulting cost--accuracy trade-off against shared-reward GRPO under a one-epoch budget. The figure indicates the practical overhead of counterfactual evaluation, so we treat it as a cost diagnostic rather than a compute-normalized superiority claim.

	\subsubsection{Reward Distribution}
	\cref{fig:boxPic} illustrates that shared rewards assign identical signals to both roles, whereas CCPO rewards better separate helpful, redundant, and harmful Thinker contributions. This visualization provides qualitative support for the intended behavior of the credit allocator: positive updates are concentrated on cases where the removed role changes the verifier outcome.

	\section{Conclusion}
	We present CCPO and SEPO, two optimizer-agnostic credit assignment methods for collaborative LLM training under sparse joint rewards. CCPO uses counterfactual outcomes to estimate marginal contribution, while SEPO uses verifier-anchored self and peer evaluations to redistribute credit. They operate at the reward-construction layer and can be paired with policy optimizers such as GRPO and GSPO. Experiments on mathematical reasoning benchmarks show that explicit credit assignment can improve over shared-reward training, especially on MATH500 and selected out-of-distribution benchmarks, while some datasets still favor shared rewards. These results highlight credit assignment as a key design axis for collaborative LLM systems and motivate future extensions to richer interaction graphs, stronger baselines, and process-level rewards.

	\section*{Limitations}
	CCPO and SEPO are evaluated on mathematical reasoning benchmarks with a two-agent Think--Solve topology. Extending them to richer interaction graphs, longer multi-turn collaboration, process-level rewards, and unreliable or adversarial collaborators remains future work. Our experiments report single-run exact-match accuracy and focus on shared-reward baselines, so small differences and comparisons to other credit-assignment alternatives should be interpreted cautiously. CCPO also requires additional verifier calls, and SEPO depends on well-calibrated self/peer rubrics.
	
	\section*{Ethics Statement}
	This paper develops credit assignment methods for collaborative LLM training. The main risks are those associated with stronger automated reasoning and coordination, including misuse in complex task automation or persuasive generation. Because self/peer scoring can be noisy or exploitable, reliable external verification, calibrated rubrics, and transparent role attribution are important safeguards before deployment beyond controlled reasoning benchmarks.
	
	\bibliography{custom}
	\newpage
	\appendix
	\onecolumn

	\section{Related Work}
	\label{app:related_work}
	
	\subsection{Credit Assignment}
	
	Credit assignment is a core challenge in multi-agent RL, particularly with sparse or delayed joint rewards where agent behaviors are tightly coupled \cite{wang2025xmix,fu2024closely}. A standard paradigm is centralized training with decentralized execution (CTDE), which leverages global information during learning while maintaining independent policies at deployment \cite{li2021shapley}. Prior methods largely fall into two lines. Value decomposition approaches, e.g., QMIX \citep{rashid2020weighted}, learn structured mappings that combine per-agent utilities into a team value, but they typically rely on monotonicity or factorization assumptions that are hard to justify for language collaboration, where an agent's message can have non-monotone and highly context-dependent effects. Counterfactual approaches estimate marginal contributions by comparing realized outcomes to hypothetical alternatives. Shapley-value formulations are principled but often computationally prohibitive due to coalitional evaluation. Related directions, such as reward redistribution across agents and time \cite{kapoor2024agent} or curriculum-based counterfactual advantages \cite{jin2025curriculum}, are not tailored to long, discrete LLM generation trajectories and typically require extra rollouts or repeated reward evaluations.
	
	These limitations motivate a practical tension in collaborative LLM training: counterfactual credit is desirable for resolving free-riding, yet naive counterfactual estimation is too expensive for long-sequence generation. CCPO resolves this tension by constructing lightweight, role- and topology-aware counterfactual baselines that can often be computed within the same sampling instance (or with minimal additional decoding), making counterfactual credit allocation feasible under strict generation budgets.
	
	\subsection{Reinforcement Learning for LLMs}
	
	RL has become a standard paradigm for aligning LLMs on reasoning, generation, and preference tasks \cite{zhang2026heterogeneous,huang2026does,huang2025adaptive,lu2026contextual,huang2026real}. Classic policy-gradient methods (e.g., REINFORCE \citep{williams1992simple} and Actor--Critic) optimize expected return but can suffer from high variance, while trust-region methods such as TRPO \citep{schulman2017trustregionpolicyoptimization} and PPO \citep{schulman2017proximalpolicyoptimizationalgorithms} stabilize training via constrained updates. For LLM fine-tuning, training an explicit critic is often costly and unstable, motivating critic-free approaches. GRPO estimates advantages from group statistics over multiple sampled responses, eliminating the need for a value network.
	HA-DW \citep{yang2026your} provides a principled theoretical analysis of group-based advantage estimation.
	Several extensions further improve robustness and efficiency, including GSPO, DAPO \citep{yu2025dapo}, FSPO \citep{mao2025clipsequencesfairlyenforcing}, and REINFORCE++, which refine importance weighting, clipping, or normalization for sequence-level optimization.
	
	Most of this literature focuses on single-policy optimization, where advantage estimation is the primary concern. In contrast, collaborative LLMs introduce an additional structural challenge: the reward is shared by a \emph{team}, but learning should be driven by \emph{agent-specific} signals that reflect marginal influence. CCPO and SEPO target this missing piece by transforming a shared joint reward into role-specific credit signals that can be fed into GRPO, GSPO, REINFORCE++, or other policy-gradient optimizers without requiring expensive critics or repeated re-rollouts.
	
	\subsection{Multi-Agent LLM Training}
	
	Multi-agent collaboration has been shown to improve LLM reasoning via test-time interaction mechanisms such as debate, critique, role assignment, and iterative refinement \cite{chen2025llmboostmakelargelanguage,zou2025transformer}. While effective at inference, these methods typically do not internalize collaborative behaviors into model parameters.
	
	Recent work has shifted toward training-time multi-agent learning \cite{chen2026weak,chen2025scoring}. ILR \cite{lin2025interactive}, MAPoRL \cite{park2025maporl}, MAGRPO \cite{liu2025llm}, and ReMA \cite{wan2503rema} propose multi-agent RL objectives to encourage cooperation, often relying on shared rewards or hierarchically structured signals. MARFT \cite{liao2025marft} generalizes multi-agent fine-tuning via parameter-efficient adaptation. In parallel, structured role-based systems such as MALT \cite{motwani2024malt} and CORY \cite{ma2024coevolving} design explicit pipelines or role-rotation schemes to improve coordination.
	
	Despite these advances, many existing approaches still adopt implicit credit sharing: agents are updated from a common team signal or from task-specific heuristics that do not explicitly isolate marginal causality, leaving free-riding and negative contributions largely unresolved. Our perspective is to treat credit assignment as the first-class bottleneck for collaborative LLM training. Unlike COMA-style counterfactual baselines, which are typically formulated around centralized critics and discrete action replacement, CCPO constructs task-level textual counterfactuals by removing or nulling an agent's message and re-evaluating the external verifier. Unlike Shapley-style attribution, it avoids enumerating coalitions and instead uses a lightweight leave-one-role-out comparison suited to long LLM generations, while acknowledging that leave-one-role-out is not a full Shapley estimator and can be biased when interactions are highly nonlinear. SEPO complements this counterfactual route by using structured self- and peer-evaluation as a bounded verifier-anchored credit signal. Compared with shared-reward multi-agent LLM training, our goal is not to introduce a new base optimizer, but to provide role-specific rewards that can be plugged into existing optimizers.

	\section{Theoretical Details for Counterfactual Credit}
	\label{app:ccpo_convergence}
	
	This section formalizes two properties used to support the idealized discussion in the main text.
	We first show that counterfactual credit can be viewed as an action-independent baseline for the active agent, and then state a TRPO-style lower bound for conservative block updates.
	As in prior work, PPO/GRPO are practical approximations to trust-region methods and do not guarantee strict monotonic improvement in general.
	
	Let $x\sim\mathcal{D}$ be a prompt and let $\boldsymbol{\theta}=(\theta_1,\dots,\theta_K)$ denote the parameters of $K$ agents with policies $\{\pi_{\theta_k}\}_{k=1}^K$ under a fixed collaboration protocol $\mathcal{C}$.
	Let $\tau\sim \pi_{\boldsymbol{\theta}}(\cdot\mid x)$ be the joint rollout and let $R(\tau)\in[0,1]$ be a terminal reward (e.g., $0/1$ correctness).
	The joint objective is
	\begin{equation}
		J(\boldsymbol{\theta})
		=
		\mathbb{E}_{x\sim\mathcal{D}}
		\ \mathbb{E}_{\tau\sim \pi_{\boldsymbol{\theta}}(\cdot\mid x)}
		\bigl[R(\tau)\bigr].
		\label{eq:app_joint_obj}
	\end{equation}
	
	For each agent $k$, define a counterfactual rollout $\tau^{\neg k}$ constructed by removing agent $k$ while keeping the other agents and protocol $\mathcal{C}$ unchanged, and define the counterfactual reward
	\begin{equation}
		R_{\neg k}:=R(\tau^{\neg k}).
		\label{eq:app_cf_reward}
	\end{equation}
	The counterfactual margin is
	\begin{equation}
		\Delta_k := R(\tau)-R_{\neg k}.
		\label{eq:app_delta}
	\end{equation}
	
	\subsection{Counterfactual credit as an action-independent baseline}
	The following Lemma~\ref{lem:app_cf_unbiased} shows that, under a mild conditional-independence assumption, using the counterfactual baseline $R_{\neg k}$ yields an unbiased policy-gradient estimator for agent~$k$, i.e., replacing $R(\tau)$ with $R(\tau)-R_{\neg k}$ does not introduce gradient bias.
	
	\begin{lemma}
		\label{lem:app_cf_unbiased}
		Fix an agent $k$ and hold $\theta_{-k}$ fixed.
		Assume that conditioned on $(x,\theta_{-k})$, the random variable $R_{\neg k}$ is independent of agent $k$'s sampled actions in the joint rollout $\tau$ (equivalently, $R_{\neg k}$ is measurable with respect to randomness external to agent $k$ in the current rollout).
		Then
		\begin{equation}
			\nabla_{\theta_k} J(\boldsymbol{\theta})
			=
			\mathbb{E}\!\left[
			\sum_{t}
			\nabla_{\theta_k}\log\pi_{\theta_k}(a_{k,t}\mid s_{k,t})\,
			\Delta_k
			\right],
			\label{eq:app_cf_pg}
		\end{equation}
		where $a_{k,t}$ and $s_{k,t}$ denote the (token-level or turn-level) action and state of agent $k$.
		Consequently, replacing $R(\tau)$ by $R(\tau)-R_{\neg k}$ does not introduce gradient bias for agent $k$.
	\end{lemma}
	
	\begin{proof}
		Write $b(x,\theta_{-k}):=R_{\neg k}$.
		By assumption, $b$ does not depend on agent $k$'s sampled actions in the current rollout.
		Using the log-derivative trick,
		\[
		\nabla_{\theta_k} J(\boldsymbol{\theta})
		=
		\mathbb{E}\!\left[
		\sum_t \nabla_{\theta_k}\log\pi_{\theta_k}(a_{k,t}\mid s_{k,t})\, R(\tau)
		\right].
		\]
		It remains to show that subtracting $b$ does not change the expectation.
		Conditioning on $(x,\theta_{-k},b)$ and using that $b$ is independent of $\tau_k$,
		\begin{align*}
			\mathbb{E}\!\left[
			\sum_t \nabla_{\theta_k}\log\pi_{\theta_k}(a_{k,t}\mid s_{k,t})\, b
			\right]
			&=
			\mathbb{E}\!\left[
			b \cdot \mathbb{E}\!\left[\nabla_{\theta_k}\log p_{\theta_k}(\tau_k\mid x,\theta_{-k}) \,\middle|\, x,\theta_{-k},b\right]
			\right] \\
			&=
			\mathbb{E}\!\left[
			b \cdot 0
			\right]
			=0,
		\end{align*}
		since the conditional expectation of the score $\nabla_{\theta_k}\log p_{\theta_k}(\tau_k\mid x,\theta_{-k})$ is zero.
		Therefore,
		\[
		\mathbb{E}\!\left[
		\sum_t \nabla_{\theta_k}\log\pi_{\theta_k}(a_{k,t}\mid s_{k,t})\, R(\tau)
		\right]
		=
		\mathbb{E}\!\left[
		\sum_t \nabla_{\theta_k}\log\pi_{\theta_k}(a_{k,t}\mid s_{k,t})\, (R(\tau)-b)
		\right],
		\]
		which yields Eq.~\eqref{eq:app_cf_pg}.
	\end{proof}
	
	\subsection{From bounded ratios to an effective trust region}
	
	In practice, the CCPO experiments construct role-specific rewards and then update each agent using GRPO with clipped importance ratios; we optionally monitor the empirical KL divergence to avoid overly large policy shifts.
	These mechanisms motivate an effective trust-region interpretation, but clipping alone does not imply a strict KL bound for practical sequence-level updates.
	The following Lemma \ref{lem:clip_kl} provides a sufficient condition linking idealized uniform ratio bounds to a per-state KL upper bound.
	
	\begin{lemma}
		\label{lem:clip_kl}
		Fix a state $s$ and two policies $\pi_{\text{old}}(\cdot\mid s)$ and $\pi_{\text{new}}(\cdot\mid s)$.
		Assume the likelihood ratio is uniformly bounded:
		\begin{equation}
			1-\epsilon_c \ \le\ \frac{\pi_{\text{new}}(a\mid s)}{\pi_{\text{old}}(a\mid s)} \ \le\ 1+\epsilon_c
			\quad
			\text{for all } a \text{ with } \pi_{\text{old}}(a\mid s)>0,
			\label{eq:ratio_bound_all_actions}
		\end{equation}
		where $\epsilon_c\in(0,1)$.
		Then
		\begin{equation}
			D_{\mathrm{KL}}\!\bigl(\pi_{\text{old}}(\cdot\mid s)\,\|\,\pi_{\text{new}}(\cdot\mid s)\bigr)
			\ \le\ -\log(1-\epsilon_c),
			\label{eq:kl_bound_from_clip}
		\end{equation}
		and consequently
		$
		D_{\mathrm{KL}}^{\max}(\pi_{\text{old}},\pi_{\text{new}})
		\le -\log(1-\epsilon_c).
		$
	\end{lemma}
	
	\begin{proof}
		By definition,
		$
		D_{\mathrm{KL}}(\pi_{\text{old}}\|\pi_{\text{new}})
		=
		\mathbb{E}_{a\sim\pi_{\text{old}}(\cdot\mid s)}
		\!\left[\log\frac{\pi_{\text{old}}(a\mid s)}{\pi_{\text{new}}(a\mid s)}\right]
		=
		\mathbb{E}_{a\sim\pi_{\text{old}}(\cdot\mid s)}[-\log r(a)],
		$
		where $r(a)=\pi_{\text{new}}(a\mid s)/\pi_{\text{old}}(a\mid s)$.
		Using $r(a)\ge 1-\epsilon_c$ gives $-\log r(a)\le -\log(1-\epsilon_c)$ for all $a$, hence the bound.
	\end{proof}
	
	\subsection{Idealized monotonic improvement under block updates}
	We now state a TRPO-style monotonic improvement bound for a single block update.
	We present the result for a discounted MDP with $\gamma\in(0,1)$; the episodic finite-horizon analogue follows similarly.
	\begin{theorem}
		\label{thm:app_block_trpo}
		Fix an iteration $t$ and an active agent $k$.
		Holding $\theta_{-k}^t$ fixed induces a stationary MDP for agent $k$.
		Let $\pi_{\text{old}}:=\pi_{\theta_k^t}$ and let $\pi$ be a candidate new policy for agent $k$ in this induced MDP.
		Let $A_{\pi_{\text{old}}}(s,a)$ be the advantage of $\pi_{\text{old}}$ and assume it is bounded:
		\begin{equation}
			\epsilon_t := \max_{s,a}\bigl|A_{\pi_{\text{old}}}(s,a)\bigr| < \infty.
			\label{eq:app_eps_def}
		\end{equation}
		Define the TRPO surrogate objective
		\begin{equation}
			L_{\pi_{\text{old}}}(\pi)
			:=
			J(\theta_k^{t},\theta_{-k}^{t})
			+
			\frac{1}{1-\gamma}\;
			\mathbb{E}_{s\sim d_{\pi_{\text{old}}},\ a\sim \pi(\cdot\mid s)}
			\!\left[A_{\pi_{\text{old}}}(s,a)\right],
			\label{eq:app_surrogate}
		\end{equation}
		where $d_{\pi_{\text{old}}}$ is the discounted state visitation distribution of $\pi_{\text{old}}$.
		Let
		\begin{equation}
			D_{\mathrm{KL}}^{\max}(\pi_{\text{old}},\pi)
			:=
			\max_s D_{\mathrm{KL}}\!\bigl(\pi_{\text{old}}(\cdot\mid s)\,\|\,\pi(\cdot\mid s)\bigr).
			\label{eq:app_klmax}
		\end{equation}
		Suppose the block update outputs $\pi_{\text{new}}$ such that
		\begin{align}
			D_{\mathrm{KL}}^{\max}(\pi_{\text{old}},\pi_{\text{new}}) \le \delta_t
			\label{eq:app_tr_cond1}
		\end{align}
		and
		\begin{align}
			L_{\pi_{\text{old}}}(\pi_{\text{new}}) \ge L_{\pi_{\text{old}}}(\pi_{\text{old}}) + \Delta_t .
			\label{eq:app_tr_cond2}
		\end{align}
		Then
		\begin{equation}
			J(\theta_k^{t+1},\theta_{-k}^{t})
			-
			J(\theta_k^{t},\theta_{-k}^{t})
			\ \ge\
			\Delta_t
			-
			C(\gamma)\,\epsilon_t\,\sqrt{\delta_t},
			\qquad
			C(\gamma):=\frac{2\gamma\sqrt{2}}{(1-\gamma)^2}.
			\label{eq:app_gain_bound}
		\end{equation}
		In particular, if $\Delta_t \ge C(\gamma)\epsilon_t\sqrt{\delta_t}$, then
		$J(\theta_k^{t+1},\theta_{-k}^{t}) \ge J(\theta_k^{t},\theta_{-k}^{t})$.
		If additionally $J(\boldsymbol{\theta})\in[0,1]$, then under such idealized block updates the sequence
		$\{J(\boldsymbol{\theta}^{t})\}_{t\ge 0}$ is non-decreasing and convergent.
	\end{theorem}
	
	\begin{proof}
		A standard TRPO analysis yields the lower bound (via the performance-difference lemma plus occupancy perturbation control)
		\begin{equation}
			J(\theta_k,\theta_{-k}^{t})
			\ \ge\
			L_{\pi_{\text{old}}}(\pi_{\theta_k})
			-
			\frac{4\gamma}{(1-\gamma)^2}\,\epsilon_t\,
			D_{\mathrm{TV}}^{\max}(\pi_{\text{old}},\pi_{\theta_k}),
			\label{eq:app_trpo_lb_tv}
		\end{equation}
		where
		$D_{\mathrm{TV}}^{\max}(\pi_{\text{old}},\pi)=\max_s D_{\mathrm{TV}}(\pi_{\text{old}}(\cdot\mid s),\pi(\cdot\mid s))$.
		By Pinsker's inequality,
		$D_{\mathrm{TV}}(p,q)\le \sqrt{\tfrac12 D_{\mathrm{KL}}(p\|q)}$,
		we have
		\[
		D_{\mathrm{TV}}^{\max}(\pi_{\text{old}},\pi_{\text{new}})
		\le
		\sqrt{\tfrac12 D_{\mathrm{KL}}^{\max}(\pi_{\text{old}},\pi_{\text{new}})}
		\le
		\sqrt{\tfrac12 \delta_t}.
		\]
		Substituting into Eq.~\eqref{eq:app_trpo_lb_tv} gives
		\[
		J(\theta_k^{t+1},\theta_{-k}^{t})
		\ge
		L_{\pi_{\text{old}}}(\pi_{\text{new}})
		-
		C(\gamma)\epsilon_t\sqrt{\delta_t}.
		\]
		Using Eq.~\eqref{eq:app_tr_cond2} and the identity
		$L_{\pi_{\text{old}}}(\pi_{\text{old}})=J(\theta_k^{t},\theta_{-k}^{t})$
		yields Eq.~\eqref{eq:app_gain_bound}.
	\end{proof}
	
	\begin{proof}[Proof of Theorem~\ref{prop:ccpo_tr_mono}]
		Fix an iteration $t$ and the active agent $k$.
		During this block update, the other agents are fixed at $\theta_{-k}^{t}$ and the protocol $\mathcal{C}$ is fixed, so the interaction defines an induced stationary MDP for agent $k$.
		Let $\pi_{\text{old}}:=\pi_{\theta_k^{t}}$ and $\pi_{\text{new}}:=\pi_{\theta_k^{t+1}}$ be the policy before and after the block update.
		The conditions in Theorem~\ref{prop:ccpo_tr_mono} correspond to Eqs.~\eqref{eq:app_tr_cond1}--\eqref{eq:app_tr_cond2} (with $\delta_t=\delta$ and $\Delta_t=\Delta_L$), and Theorem~\ref{thm:app_block_trpo} implies
		\[
		J(\theta_k^{t+1},\theta_{-k}^{t}) - J(\theta_k^{t},\theta_{-k}^{t})
		\ge
		\Delta_L - C(\gamma)\,\epsilon\,\sqrt{\delta},
		\]
		which matches Eq.~\eqref{eq:main_trpo_gain}.
		If $\Delta_L \ge C(\gamma)\epsilon\sqrt{\delta}$ then the block update is non-decreasing.
		Since $R(\tau)\in[0,1]$ implies $J(\boldsymbol{\theta})\in[0,1]$, any non-decreasing sequence of objective values is bounded above and thus convergent.
	\end{proof}
	
	Finally, Lemma~\ref{lem:app_cf_unbiased} applies to the Think--Solve instantiation in Appendix~\ref{app:ccpo_details} as long as, for the active agent $k$, the counterfactual term $R_{\neg k}$ (and any shaping terms used to compute its advantage, such as running statistics and normalizers) is conditionally independent of agent $k$'s sampled actions in the current rollout.
	In Think--Solve, $R_{\mathrm{solo}}=R_{\neg 1}$ depends only on $(x,\theta_2)$ and independent sampling from the Solver policy $\pi_{\theta_2}(\cdot\mid x)$, so the counterfactual term acts as an action-independent baseline for the active agent within the current rollout.

	\subsection{Counterfactual credit versus shared terminal rewards}
	\label{app:cf_vs_shared}
	
	This subsection provides a complete proof of Theorem~\ref{prop:cf_vs_shared_main} in the main text by decomposing it into three standard steps: (i) unbiasedness under action-independent baselines, (ii) the variance-optimal scalar baseline, and (iii) a sufficient condition under which a counterfactual baseline improves upon shared rewards.
	Throughout, fix an active agent $k$ and hold $\theta_{-k}$ fixed.
	Let $R(\tau)\in[0,1]$ be the terminal reward for a joint rollout $\tau$, and define the joint objective $J$ as in Eq.~\eqref{eq:app_joint_obj}.
	Let $a_{k,t}$ and $s_{k,t}$ denote the (token-level or turn-level) action and state of agent $k$, and define
	\[
	g_k(\tau_k)\;:=\;\sum_t \nabla_{\theta_k}\log \pi_{\theta_k}(a_{k,t}\mid s_{k,t})
	\]
	as agent $k$'s score-function term.
	
	\begin{lemma}[Action-independent baselines do not change the policy gradient]
		\label{lem:baseline_unbiased_app}
		Let $b=b(x,\theta_{-k})$ be any random variable that is measurable with respect to $(x,\theta_{-k})$ and independent of agent $k$'s sampled actions in the current rollout (conditioned on $(x,\theta_{-k})$).
		Then
		\begin{equation}
			\label{eq:baseline_unbiased_app}
			\mathbb{E}\!\left[g_k(\tau_k)\,b\right] \;=\; 0,
		\end{equation}
		and hence
		\begin{equation}
			\nabla_{\theta_k}J(\boldsymbol{\theta})
			=
			\mathbb{E}\!\left[g_k(\tau_k)\,\bigl(R(\tau)-b\bigr)\right].
		\end{equation}
	\end{lemma}
	
	\begin{proof}
		Condition on $(x,\theta_{-k},b)$.
		By the assumed action-independence, $b$ is constant with respect to the randomness of $\tau_k$ under $\pi_{\theta_k}$, so
		\[
		\mathbb{E}\!\left[g_k(\tau_k)\,b \,\middle|\, x,\theta_{-k},b\right]
		=
		b\cdot \mathbb{E}\!\left[g_k(\tau_k)\,\middle|\,x,\theta_{-k},b\right]
		=
		b\cdot \mathbb{E}\!\left[\nabla_{\theta_k}\log p_{\theta_k}(\tau_k\mid x,\theta_{-k}) \,\middle|\,x,\theta_{-k},b\right]
		=0,
		\]
		since the conditional expectation of the score is zero.
		Taking expectation over $(x,\theta_{-k},b)$ yields Eq.~\eqref{eq:baseline_unbiased_app}.
	\end{proof}
	
	We now characterize the variance-optimal scalar baseline within this unbiased family.
	
	\begin{lemma}[Optimal scalar baseline for variance reduction]
		\label{lem:optimal_baseline_app}
		Fix $(x,\theta_{-k})$ and consider estimators $\widehat{G}_b = g_k(\tau_k)\,(R(\tau)-b)$ where $b$ is any scalar baseline measurable w.r.t.\ $(x,\theta_{-k})$ and action-independent for agent $k$.
		Then the conditional variance $\mathrm{Var}(\widehat{G}_b\mid x,\theta_{-k})$ is minimized by
		\begin{equation}
			\label{eq:opt_baseline_app}
			b^\star(x,\theta_{-k})
			\;=\;
			\frac{\mathbb{E}\!\left[\|g_k(\tau_k)\|_2^2\,R(\tau)\,\middle|\,x,\theta_{-k}\right]}
			{\mathbb{E}\!\left[\|g_k(\tau_k)\|_2^2\,\middle|\,x,\theta_{-k}\right]}.
		\end{equation}
		Moreover, for any two baselines $b_1,b_2$ in this class,
		\begin{equation}
			\label{eq:var_decomp_app}
			\mathrm{Var}(\widehat{G}_{b_1}\mid x,\theta_{-k})
			-
			\mathrm{Var}(\widehat{G}_{b_2}\mid x,\theta_{-k})
			=
			\mathbb{E}\!\left[\|g_k(\tau_k)\|_2^2\Bigl((b_1-b^\star)^2-(b_2-b^\star)^2\Bigr)\,\middle|\,x,\theta_{-k}\right].
		\end{equation}
	\end{lemma}
	
	\begin{proof}
		For fixed $(x,\theta_{-k})$, expand the conditional second moment:
		\[
		\mathbb{E}\!\left[\|\widehat{G}_b\|_2^2 \,\middle|\,x,\theta_{-k}\right]
		=
		\mathbb{E}\!\left[\|g_k(\tau_k)\|_2^2\,(R(\tau)-b)^2 \,\middle|\,x,\theta_{-k}\right].
		\]
		This is a convex quadratic function of $b$ with derivative
		\[
		\frac{\partial}{\partial b}
		\mathbb{E}\!\left[\|g_k(\tau_k)\|_2^2\,(R(\tau)-b)^2 \,\middle|\,x,\theta_{-k}\right]
		=
		-2\,\mathbb{E}\!\left[\|g_k(\tau_k)\|_2^2\,(R(\tau)-b)\,\middle|\,x,\theta_{-k}\right].
		\]
		Setting the derivative to zero yields Eq.~\eqref{eq:opt_baseline_app}.
		To obtain Eq.~\eqref{eq:var_decomp_app}, write $(R-b)^2=(R-b^\star)^2+(b-b^\star)^2-2(R-b^\star)(b-b^\star)$ and use the optimality condition $\mathbb{E}[\|g_k\|_2^2(R-b^\star)\mid x,\theta_{-k}]=0$.
	\end{proof}
	
	Lemma~\ref{lem:optimal_baseline_app} yields an immediate comparison between shared rewards and counterfactual credit.
	
	\begin{corollary}[Shared rewards as a suboptimal baseline; sufficient condition for improvement]
		\label{cor:cf_improves_shared_app}
		The shared-reward estimator corresponds to $b\equiv 0$.
		If a counterfactual term $R_{\neg k}$ is action-independent for agent $k$ and satisfies
		\[
		\mathbb{E}\!\left[\|g_k(\tau_k)\|_2^2\,(R_{\neg k}-b^\star)^2 \,\middle|\,x,\theta_{-k}\right]
		\ \le\
		\mathbb{E}\!\left[\|g_k(\tau_k)\|_2^2\,(0-b^\star)^2 \,\middle|\,x,\theta_{-k}\right],
		\]
		then
		\[
		\mathrm{Var}\!\left(g_k(\tau_k)\,(R(\tau)-R_{\neg k}) \,\middle|\,x,\theta_{-k}\right)
		\ \le\
		\mathrm{Var}\!\left(g_k(\tau_k)\,R(\tau) \,\middle|\,x,\theta_{-k}\right).
		\]
	\end{corollary}
	
	\paragraph{Proof of Theorem~\ref{prop:cf_vs_shared_main}.}
	Fix agent $k$ and hold $\theta_{-k}$ fixed.
	Under the assumption of Theorem~\ref{prop:cf_vs_shared_main}, the counterfactual term $R_{\neg k}$ is measurable w.r.t.\ $(x,\theta_{-k})$ and is independent of agent $k$'s sampled actions in the current rollout (conditioned on $(x,\theta_{-k})$).
	Applying Lemma~\ref{lem:baseline_unbiased_app} with $b=R_{\neg k}$ yields $\mathbb{E}[g_k(\tau_k)R_{\neg k}]=0$ and hence
	$\nabla_{\theta_k}J(\boldsymbol{\theta})=\mathbb{E}[g_k(\tau_k)(R(\tau)-R_{\neg k})]=\mathbb{E}[g_k(\tau_k)\Delta_k]$,
	which proves the ``no gradient bias'' claim.
	
	Next, consider the family of estimators $\widehat{G}_b=g_k(\tau_k)(R(\tau)-b)$ where $b$ is any scalar baseline measurable w.r.t.\ $(x,\theta_{-k})$ and action-independent for agent $k$.
	Lemma~\ref{lem:optimal_baseline_app} shows that, among this unbiased family, the conditional variance $\mathrm{Var}(\widehat{G}_b\mid x,\theta_{-k})$ is minimized by $b^\star$ given in Eq.~\eqref{eq:opt_baseline_app}, which coincides with Eq.~\eqref{eq:opt_baseline_main} in the main text.
	
	Finally, the shared-reward estimator corresponds to $b\equiv 0$, while the counterfactual-credit estimator corresponds to $b=R_{\neg k}$.
	By Corollary~\ref{cor:cf_improves_shared_app}, if $R_{\neg k}$ is closer to $b^\star$ than $0$ in the weighted mean-square sense (the condition stated in Theorem~\ref{prop:cf_vs_shared_main}), then
	$\mathrm{Var}(g_k(\tau_k)\Delta_k\mid x,\theta_{-k}) \le \mathrm{Var}(g_k(\tau_k)R(\tau)\mid x,\theta_{-k})$.
	This concludes the proof.
	\hfill$\square$
	
	In the CCPO instantiation used in this paper, the action-independence condition holds for the active agent.
	Under Think--Solve, the counterfactual $R_{\mathrm{solo}}=R_{\neg 1}$ is obtained by independently sampling the Solver from $\pi_{\theta_2}(\cdot\mid x)$ without conditioning on the Thinker's output, which is likewise independent of the Thinker's sampled actions in the joint rollout.
	Therefore, the counterfactual credit $\Delta_k=R-R_{\neg k}$ can be viewed as an action-independent baseline subtraction for the active agent, preserving unbiasedness and enabling variance reduction relative to shared rewards whenever $R_{\neg k}$ is closer to the optimal baseline than $0$.

	
	\section{Detailed Credit Construction for CCPO and SEPO}
	\label{app:ccpo_details}
	
	
	\subsection{Unified Trajectory View and Counterfactual Construction}
	\label{app:unified_view}

	For each prompt $x\sim\mathcal{D}$, a collaboration topology induces a joint generation process over $K$ agents.
	We denote the $j$-th joint rollout by $\tau^{(j)}=(y_1^{(j)},\ldots,y_K^{(j)})$ with reward
	$R_{\mathrm{joint}}^{(j)} := R(\tau^{(j)})$.
	
	CCPO associates each agent $i$ with a counterfactual trajectory $\tau^{(j),\neg i}$, which removes agent $i$'s contribution while keeping the remaining agents fixed under the same sampling instance.
	Evaluating this trajectory yields the counterfactual reward
	$R_{\neg i}^{(j)} := R(\tau^{(j),\neg i})$.
	The marginal contribution is then
	\begin{equation*}
		\Delta_i^{(j)} \;=\; R_{\mathrm{joint}}^{(j)} - R_{\neg i}^{(j)}.
	\end{equation*}
	Here $R_{\neg i}^{(j)}$ is the same canonical counterfactual reward used in the main text; notation such as $R_{\mathrm{solo}}$ below denotes a topology-specific instance of $R_{\neg i}^{(j)}$.
	The remainder of the appendix instantiates $\tau^{(j),\neg i}$ and $R_{\neg i}^{(j)}$ for the Think--Solve topology used in the paper, and specifies the resulting shaped rewards and advantages.
	
	\subsection{The algorithm details of CCPO}
	\label{app:dual_agent_details}
	
	We consider $K=2$ agents. For each prompt $x$, we first sample $N$ cooperative rollouts:
	\begin{align*}
		y_1^{(j)} \sim \pi_{\theta_1}(\cdot\mid x), \quad y_2^{(j)} &\sim \pi_{\theta_2}(\cdot\mid x, y_1^{(j)}), \quad j=1,\ldots,N .
	\end{align*}
	The joint reward is
	\begin{equation*}
		R_{\mathrm{joint}}^{(j)} \;:=\; R(x, y_1^{(j)}, y_2^{(j)}).
	\end{equation*}
	
	To construct the counterfactual for the Thinker, we additionally sample $N$ rollouts where the Solver answers without access to $y_1$:
	\begin{equation}
		y_{2,\text{solo}}^{(j)} \sim \pi_{\theta_2}(\cdot\mid x),
		\qquad
		R_{\neg 1}^{(j)} \equiv R_{\mathrm{solo}}^{(j)} \;:=\; R(x, \varnothing, y_{2,\text{solo}}^{(j)}).
		\label{eq:r_solo_dual}
	\end{equation}
	The marginal contribution attributed to the Thinker is
	\begin{equation*}
		\Delta_1^{(j)} \;:=\; R_{\mathrm{joint}}^{(j)} - R_{\neg 1}^{(j)}
		\;=\; R_{\mathrm{joint}}^{(j)} - R_{\mathrm{solo}}^{(j)} ,
	\end{equation*}
	where $R_{\mathrm{solo}}^{(j)}$ is only a Think--Solve shorthand for the canonical $R_{\neg 1}^{(j)}$.
	
	Next we convert $\Delta_1^{(j)}$ into a bounded shaped reward and a within-prompt advantage.
	We maintain EMA statistics for $\Delta$:
	\begin{align*}
		\mu_{\Delta}^{(t)} = \lambda\,\mu_{\Delta}^{(t-1)} + (1-\lambda)\,\mu_{\Delta}^{\text{batch}}, \quad (\sigma_{\Delta}^2)^{(t)} &= \lambda\,(\sigma_{\Delta}^2)^{(t-1)} + (1-\lambda)\,(\sigma_{\Delta}^2)^{\text{batch}} .
	\end{align*}
	We then normalize and shape
	\begin{align*}
		z_{\Delta}^{(j)} = \frac{\Delta_1^{(j)} - \mu_{\Delta}}{\sigma_{\Delta} + \epsilon}, \quad r_1^{(j)} = \tanh\!\big(\alpha \, z_{\Delta}^{(j)}\big),
	\end{align*}
	and compute the within-prompt advantage
	\begin{equation*}
		A_1^{(j)} \;=\;
		\frac{r_1^{(j)} - \overline{r}_1}{\operatorname{std}(r_1)+\epsilon},
		\qquad
		\overline{r}_1=\frac{1}{N}\sum_{j=1}^{N}r_1^{(j)}.
	\end{equation*}
	
	For the Solver, we use a fused signal that balances joint performance with independent robustness.
	We maintain EMA statistics for $R_{\mathrm{joint}}$ and $R_{\mathrm{solo}}$:
	\begin{align*}
		\mu_{\text{joint}}^{(t)} &= \lambda\, \mu_{\text{joint}}^{(t-1)} + (1-\lambda)\,\mu_{\text{joint}}^{\text{batch}}, \quad (\sigma_{\text{joint}}^2)^{(t)} = \lambda\, (\sigma_{\text{joint}}^2)^{(t-1)} + (1-\lambda)\,(\sigma_{\text{joint}}^2)^{\text{batch}}, \\
		\mu_{\text{solo}}^{(t)} &= \lambda\, \mu_{\text{solo}}^{(t-1)} + (1-\lambda)\, \mu_{\text{solo}}^{\text{batch}}, \quad (\sigma_{\text{solo}}^2)^{(t)} = \lambda\, (\sigma_{\text{solo}}^2)^{(t-1)} + (1-\lambda)\,(\sigma_{\text{solo}}^2)^{\text{batch}} .
	\end{align*}
	We normalize each reward stream as
	\begin{align*}
		z_{\text{joint}}^{(j)} = \frac{R_{\mathrm{joint}}^{(j)} - \mu_{\text{joint}}}{\sigma_{\text{joint}} + \epsilon}, \quad z_{\text{solo}}^{(j)} = \frac{R_{\mathrm{solo}}^{(j)} - \mu_{\text{solo}}}{\sigma_{\text{solo}} + \epsilon}.
	\end{align*}
	
	We then define a trust coefficient $g\in(0,1)$ from the historical marginal contribution, so that the update relies more on $R_{\mathrm{joint}}$ when the Thinker has been helpful and falls back toward $R_{\mathrm{solo}}$ otherwise:
	\begin{equation*}
		g \;=\; \sigma\!\left(\eta \cdot \frac{\mu_{\Delta}}{\sigma_{\Delta} + \epsilon}\right),
		\qquad
		\sigma(u)=\frac{1}{1+e^{-u}}.
	\end{equation*}
	Finally, the Solver uses the fused score and within-prompt advantage
	\begin{align*}
		r_2^{(j)} = g \cdot z_{\text{joint}}^{(j)} + (1-g)\cdot z_{\text{solo}}^{(j)}, \quad A_2^{(j)} =
		\frac{r_2^{(j)} - \overline{r}_2}{\operatorname{std}(r_2)+\epsilon}, \quad
		\overline{r}_2=\frac{1}{N}\sum_{j=1}^{N}r_2^{(j)}.
	\end{align*}
	
	\subsection{The algorithm details of SEPO}
	\label{app:sepo_details}
	
	SEPO uses the same Think--Solve rollouts but replaces counterfactual verifier calls with bounded self and peer assessments. For each rollout, the external verifier produces $R_{\mathrm{ver}}^{(j)}\in\{-1,+1\}$, and each role reports a self score and a peer score from the finite rubric $\mathcal{V}$. The fused role scores are
	\begin{align*}
		s_1^{(j)} &= \eta\,p_1^{\mathrm{self},(j)} + (1-\eta)\,p_2^{\mathrm{peer},(j)}, \\
		s_2^{(j)} &= \eta\,p_2^{\mathrm{self},(j)} + (1-\eta)\,p_1^{\mathrm{peer},(j)}.
	\end{align*}
	They are converted into normalized role weights
	\begin{equation*}
		w_i^{(j)}=\frac{s_i^{(j)}}{s_1^{(j)}+s_2^{(j)}+\epsilon}.
	\end{equation*}
	When group centering is used, SEPO defines
	\begin{equation*}
		\mathrm{bonus}_i^{(j)}
		=w_i^{(j)}-\operatorname{mean}_{j'\in\mathcal{G}(x)}\bigl(w_i^{(j')}\bigr).
	\end{equation*}
	The final role reward is anchored to the verifier:
	\begin{equation*}
		r_i^{(j)} =
		\begin{cases}
			R_{\mathrm{ver}}^{(j)} + \lambda_{\mathrm{credit}}\,\mathrm{bonus}_i^{(j)}, & R_{\mathrm{ver}}^{(j)}=+1,\\
			R_{\mathrm{ver}}^{(j)} - \lambda_{\mathrm{blame}}\,\mathrm{bonus}_i^{(j)}, & R_{\mathrm{ver}}^{(j)}=-1.
		\end{cases}
	\end{equation*}
	This construction keeps correctness as the dominant signal while allowing self and peer judgments to redistribute credit within a bounded range.
	
	\section{Hyperparameter Settings for The Experiments}
	We conducted experiments on 6 NVIDIA A800 GPUs with the hyperparameter settings in \cref{tab:hyperparameters}. Unless otherwise stated, the reported runs use GRPO as the base optimizer, but the same credit signals are designed to be reusable in GSPO, REINFORCE++, and related policy-gradient methods.
	\begin{table*}[t]
		\centering
		\caption{Training and reward-shaping hyperparameters.}
		\label{tab:hyperparameters}
		\renewcommand{\arraystretch}{1.05}
		\setlength{\tabcolsep}{6pt}
		\begin{tabular}{ccc}
			\toprule[1.1pt]
			\textbf{Category} & \textbf{Hyperparameter} & \textbf{Value} \\
			\midrule
			\multirow{6}{*}{Policy Optimization} 
			& Learning rate & $1\times 10^{-6}$ \\
			& Batch size & 64 \\
			& Samples per prompt ($n$) & 4 \\
			& Clip ratio ($\epsilon$) & 0.2 \\
			& Gradient clip & 1.0 \\
			\midrule
			\multirow{4}{*}{Reward Shaping} 
			& Contribution sensitivity ($\alpha$) & 1.0 \\
			& Gate sharpness ($\eta$) & 1.0 \\
			& EMA decay ($\lambda$) & 0.99 \\
			& Min samples for normalization & 50 \\
			\bottomrule[1.1pt]
		\end{tabular}
		\vspace{-2mm}
	\end{table*}
	\vfill

\end{document}